\documentclass[letterpaper]{article} 
\usepackage{aaai2026}  
\usepackage{times}  
\usepackage{helvet}  
\usepackage{courier}  
\usepackage[hyphens]{url}  
\usepackage{graphicx} 
\usepackage{natbib}  
\usepackage{caption} 
\usepackage{algorithm}
\usepackage{algorithmic}

\usepackage{bm}
\usepackage{stmaryrd}
\usepackage{tabularx}
\usepackage{multirow}
\usepackage{caption}
\usepackage{pifont}
\usepackage{amssymb}
\usepackage{amsthm}
\usepackage{mathtools}
\usepackage{amsmath}
\usepackage[capitalise,noabbrev]{cleveref}
\usepackage{dsfont}
\usepackage{tabularx}
\usepackage{booktabs}
\usepackage{enumitem}
\usepackage{newfloat}
\usepackage{listings}
\DeclareCaptionStyle{ruled}{labelfont=normalfont,labelsep=colon,strut=off} 
\floatstyle{ruled}
\newfloat{listing}{tb}{lst}{}
\floatname{listing}{Listing}
\newcommand{\bE}{\mathbb{E}}

\newcommand{\cD}{\mathcal{D}}
\newcommand{\cH}{\mathcal{H}}
\newcommand{\cX}{\mathcal{X}}
\newcommand{\cY}{\mathcal{Y}}

\newcommand{\KL}{\operatorname{KL}}
\newcommand{\TU}{\operatorname{TU}}
\newcommand{\AU}{\operatorname{AU}}
\newcommand{\EU}{\operatorname{EU}}
\newcommand{\bma}{\bar{\vtheta}}
\newcommand{\btheta}{\bar{\theta}}
\newcommand{\vx}{\bm{x}}
\newcommand{\htheta}{\hat{\theta}}
\newcommand{\vtheta}{\bm{\theta}}

\newcommand{\hvtheta}{\hat{\bm{\theta}}}
\newcommand{\sumK}{\sum_{k=1}^K}
\newcommand{\sumKK}{\sum_{k'=1}^K}
\newcommand{\ksimplex}{\Delta_K}
\newcommand{\ksimplextwo}[1][K]{\Delta_{#1}^{(2)}}
\newcommand{\on}[1]{\operatorname{#1}}
\DeclareMathOperator*{\argmax}{arg\,max}

\renewcommand{\vec}[1]{\boldsymbol{#1}}

\definecolor{log}{HTML}{FF7F0E}
\definecolor{zero-one}{HTML}{9467BD}
\definecolor{brier}{HTML}{2CA02C}
\definecolor{spherical}{HTML}{D62728}

\DeclarePairedDelimiter\floor{\lfloor}{\rfloor}

\newcommand*{\defeq}{\mathrel{\vcenter{\baselineskip0.5ex \lineskiplimit0pt
			\hbox{\footnotesize.}\hbox{\footnotesize.}}}%
	=}

\newcommand{\dataset}[1]{\textsc{\lowercase{\mbox{#1}}}}

\theoremstyle{plain}
\newtheorem{theorem}{Theorem}
\newtheorem*{theorem*}{Theorem}
\newtheorem*{corollary*}{Corollary}
\newtheorem{proposition}[theorem]{Proposition}
\newtheorem*{proposition*}{Proposition}

\theoremstyle{definition}

\newtheorem*{definition*}{Definition}

\theoremstyle{remark}

\title{Uncertainty Quantification for Machine Learning: \\ One Size Does Not Fit All}
\author{
    Paul Hofman\textsuperscript{\rm 1,\rm 2},
    Yusuf Sale\textsuperscript{\rm 1,\rm 2},
    Eyke Hüllermeier\textsuperscript{\rm 1,\rm 2,\rm 3}
}
\affiliations{
    \textsuperscript{\rm 1}LMU Munich\\
    \textsuperscript{\rm 2}Munich Center for Machine Learning (MCML)\\
    \textsuperscript{\rm 3}German Research Center for Artificial Intelligence (DFKI, DSA)\\
    paul.hofman@ifi.lmu.de, yusuf.sale@ifi.lmu.de, eyke@ifi.lmu.de
}

\usepackage{bibentry}

\begin{document}

\maketitle

\begin{abstract}
Proper quantification of predictive uncertainty is essential for the use of machine learning in safety-critical applications. Various uncertainty measures have been proposed for this purpose, typically claiming superiority over other measures. In this paper, we argue that there is no single best measure. Instead, uncertainty quantification should be tailored to the specific application. To this end, we use a flexible family of uncertainty measures that distinguishes between total, aleatoric, and epistemic uncertainty of second-order distributions. These measures can be instantiated with specific loss functions, so-called proper scoring rules, to control their characteristics, and we show that different characteristics are useful for different tasks. In particular, we show that, for the task of selective prediction, the scoring rule should ideally match the task loss. On the other hand, for out-of-distribution detection, our results confirm that mutual information, a widely used measure of epistemic uncertainty, performs best. Furthermore, in an active learning setting, epistemic uncertainty based on zero-one loss is shown to consistently outperform other uncertainty measures.
\end{abstract}

\section{Introduction}\label{sec:intro}
Uncertainty quantification (UQ), the assessment of a model's uncertainty in predictive tasks, has become an increasingly prominent topic in machine learning research and practice. A common distinction is made between \emph{aleatoric} and \emph{epistemic} uncertainty \citep{hullermeier2021aleatoric}. Broadly speaking, aleatoric uncertainty originates from the inherent stochastic nature of the data-generating process, while epistemic uncertainty is due to the learner's incomplete knowledge of this process. The latter can therefore be reduced by acquiring additional information, such as more training data, whereas aleatoric uncertainty, as a characteristic of the data-generating process, is non-reducible. 

Due to inherent challenges in representing epistemic uncertainty, higher-order formalisms, most notably second-order distributions (i.e., distributions over distributions), are typically employed. Given a suitable uncertainty \emph{representation}, the key question that follows is how to appropriately quantify (total) uncertainty in terms of a numerical value, and how to decompose it into an aleatoric and an epistemic component. This choice has important consequences for downstream tasks, e.g., it determines which examples are abstained on, which inputs are flagged as out-of-distribution, or which unlabeled points are queried, and poor uncertainty quantification can thus mask true performance or misguide decisions even if the base predictor is strong.

For second-order uncertainty representations, entropy‐based measures have long been the default choice \citep{depeweg2018decomposition}. Yet recent work \citep{wimmer2023quantifying} questions whether these metrics truly satisfy the core criteria of sound uncertainty quantification. Consequently, exploring alternative uncertainty measures is a natural and necessary step toward overcoming the limitations of existing approaches. However, much of the prior work treats these alternatives as general competitors rather than asking which uncertainty measure is appropriate for a specific downstream objective, leading to conflicting or opaque conclusions.

Uncertainty measures in the machine learning literature have largely been treated as a one-size-fits-all solution, with little emphasis on adapting to specific tasks. However, recent work  \citep{muscanyiBenchmarkingUncertainties2024} suggests that different tasks may require tailored uncertainty measures. Crucially, in the absence of an observable baseline, uncertainty measures are typically evaluated empirically through (downstream) tasks such as selective prediction, out-of-distribution (OoD) detection, or active learning, each of which may require different uncertainty measures.

These considerations highlight the need for a more flexible approach to uncertainty quantification, one that aligns with the specific requirements of underlying tasks. Accordingly, leveraging a classical decomposition of proper scoring rules, we adopt a loss-based family of total, aleatoric, and epistemic uncertainty measures that subsumes traditional measures as a special case \citep{saleLabelWise2024, hofman2024quantifying, kotelevskii2025risk}. 

While recent work has explored this family of measures theoretically and cast them as alternatives to the static entropy‐based approach, it has largely overlooked the most critical factor: the downstream machine learning \emph{task} used to (empirically) \emph{evaluate} the entire uncertainty pipeline. 
By tying uncertainty measures directly to each task’s evaluation loss, we demonstrate that instantiating an uncertainty measure with that same loss yields optimal alignment with the task’s objectives.
We demonstrate this both theoretically and empirically. Theoretically, we establish a formal connection between task losses and the losses used to construct uncertainty measures, showing that optimal uncertainty quantification requires alignment between these components. Empirically, we validate our framework across important downstream tasks, including selective prediction, OoD detection, and active learning, confirming that different tasks benefit from different uncertainty measures. 
Together, these results expose why common one-size-fits-all practices can be misleading and motivate more deliberate, task-aware uncertainty evaluation.

\section{Uncertainty in Machine Learning} \label{section:uq}
In this paper, we consider a standard supervised learning setting, in which a learner is given access to a set of i.i.d. training data $\cD_{\rm{train}} = \{(\vx_i, y_i)\}_{i=1}^n \in (\cX \times \cY)^n$, where $\cX$ denotes the instance space and $\cY$ the set of outcomes. We focus on the classification scenario, where $\cY = \{1, \dots, K\}$ consists of a finite set of class labels. Additionally, we denote by $\mathbb{P}(\cY)$ the set of all probability measures on $\cY$, which can be identified with the $(K-1)$-simplex $\Delta_K$. We consider a hypothesis space $\cH$, where each hypothesis $h \in \cH$ maps instances $\vec{x}$ to probability distributions on outcomes. For brevity, we write $h(\vec{x}) = \hvtheta$ for the probabilistic prediction produced by the hypothesis $h \in \cH$, where $\hvtheta = (\htheta_1, \dots, \htheta_K) \in \ksimplex$. Similarly, $\vtheta = (\theta_1, \dots, \theta_K)$ denotes the \emph{ground-truth} (conditional) probability distribution on the outcomes given a query instance $\vec{x} \in \cX$. Finally, we denote the extended real number line by $\overline{\mathbb{R}} \defeq \mathbb{R} \cup \{-\infty, +\infty\}$.

\emph{Uncertainty Representation.} A probabilistic model $h$ predicts a probability distribution that captures \emph{aleatoric} uncertainty about the outcome $y \in \cY$, but pretends full certainty about the distribution $\vtheta$ itself. In order to represent \textit{epistemic} uncertainty, we consider a Bayesian representation of uncertainty. Hence, we assume that we have access to a posterior distribution $q(h \mid \cD)$. The posterior distribution gives rise to a distribution over distributions $\vtheta$ through $$Q(\vtheta) = \int_\cH \llbracket h(\vx) = \vtheta \rrbracket \, dq(h \mid \cD) \, , 
$$ 
with $Q \in \ksimplextwo$ and $\ksimplextwo$ denotes the set of all probability distributions on $\ksimplex$ (\emph{viz.} second-order distributions). To make predictions, a representative first-order distribution is generated by model averaging $\bma = \int_{\cH} h(\vx) \, dq(h \mid \cD)$. 
In practice, we usually only have access to samples of the posterior, e.g., through an ensemble of predictors. Thus, we use a finite approximation $\bma = \frac{1}{M}\sum_{m=1}^M h^m(\vx)$, where $M$ denotes the number of ensemble members or samples drawn from the posterior in the case of e.g.\  variational inference.

\emph{Uncertainty Quantification.} Given a second-order distribution $Q \in \ksimplextwo$, the task of uncertainty quantification is to specify functionals (namely, uncertainty measures) $\TU, \AU, \EU: \ksimplextwo \rightarrow \mathbb{R}_{\geq 0}$ that quantify total, aleatoric, and epistemic uncertainty, respectively. Particularly, a well-known decomposition of proper scoring rules yields a theoretically principled family of uncertainty measures, flexibly instantiated by the choice of the loss function. 

Proper scoring rules, originating in Savage’s elicitation framework \citep{savage1971elicitation} and developed further by \citet{gneiting2007strictly}, assign numerical scores to probabilistic forecasts and incentivize \emph{truthful} reporting. A scoring rule is proper if a forecaster’s expected score is optimized exactly when the announced distribution equals their true belief, and strictly proper if this optimizer is unique. Further, a function $\cY \rightarrow \overline{\mathbb{R}}$ is called $\Delta_K$-quasi-integrable if it is measurable with respect to $2^{\mathcal{Y}}$, and is quasi-integrable with respect to all $\vtheta\in\Delta_K$. We assume scoring rules to be \emph{negatively} oriented, thus taking a standard machine learning perspective where we wish to minimize the corresponding loss. 

\begin{definition*}[Proper scoring rule]\label{def:psr}
A scoring rule is a measurable function $\ell:\Delta_K\times\cY\to\overline{\mathbb R}$ such that, for all $\hat{\vtheta}\in\Delta_K$, the expectation
\begin{equation}\label{eq:eps}
L_\ell(\hat{\vtheta},\vtheta) \defeq \mathbb{E}_{Y\sim\vtheta}\left[\ell(\hat{\vtheta},Y)\right]
\end{equation}
is well defined for every $\vtheta\in\Delta_K$. The scoring rule $\ell$ is \emph{proper} if, for all $\hat{\vtheta},\vtheta\in\Delta_K$,
\begin{equation}\label{eq:psr}
L_\ell(\vtheta,\vtheta)\ \le\ L_\ell(\hat{\vtheta},\vtheta),
\end{equation}
and \emph{strictly proper} if equality in \eqref{eq:psr} holds only when $\hat{\vtheta}=\vtheta$.
\end{definition*}

It is well-known that (strictly) proper scoring rules (and with them their corresponding expected losses) can be decomposed into a \textit{divergence} term and an \textit{entropy} term, respectively \citep{gneiting2007strictly, kullNovel2015}: 
\begin{align*}
D_\ell(\hat{\vtheta}, \vtheta)  = L_\ell(\hat{\vtheta}, \vtheta) - L_\ell(\vtheta, \vtheta), \quad  
H_\ell(\vtheta)  = L_\ell(\vtheta, \vtheta).
\end{align*}

The latter captures the expected loss that materializes even when the ground truth $\vtheta$ is predicted, whereas the former represents the ``excess loss'' that is caused by predicting $\hat{\vtheta}$ and hence deviating from the optimal prediction $\vtheta$. This decomposition naturally aligns with the distinction between \emph{irreducible} (aleatoric) and \emph{reducible} (epistemic) uncertainty: $H_\ell(\vtheta)$ is the irreducible part of the risk, and hence relates to aleatoric uncertainty, whereas $D_\ell(\hat{\vtheta}, \vtheta)$ is purely due to the learner's imperfect knowledge, or epistemic state, and could in principle be reduced by additional information. 

So far, our discussion has focused on first-order distributions, assuming access to the true conditional distribution $\vtheta$. In practice, however, and as previously motivated, uncertainty about $\vtheta$ is represented through a second-order distribution $Q$. Consequently, it is sensible to define 
\begin{align}
    \on{EU}(Q) & = \mathbb{E}_{\vtheta \sim Q}[D_\ell(\bar{\vtheta}, \vtheta)]  \label{eq:eunft} \\[0.1cm]
    & = \mathbb{E}_{\vtheta \sim Q}[L_\ell(\bar{\vtheta}, \vtheta) - L_\ell(\vtheta, \vtheta)] \\[0.1cm]
    & = \underbrace{\mathbb{E}_{\vtheta \sim Q}[L_\ell(\bar{\vtheta}, \vtheta) ]}_{\text{TU}(Q)} -  
    \underbrace{\mathbb{E}_{\vtheta \sim Q} [ L_\ell(\vtheta, \vtheta)]}_{\text{AU}(Q)}. \label{eq:tuau}
\end{align}
That is, EU is the \emph{gain}\,---\,in terms of loss reduction\,---\,the learner can expect when predicting, not on the basis of the uncertain knowledge $Q$, but only after being revealed the true $\vtheta$. Intuitively, this is plausible: The more uncertain the learner is about the true $\vtheta$ (i.e., the more dispersed $Q$), the more it can gain by getting to know this distribution. The connection to proper scoring rules is also quite obvious: Total uncertainty in (\ref{eq:tuau}) is the expected loss of the learner when predicting optimally ($\bar{\vtheta}$) on the basis of its uncertain belief $Q$. It corresponds to the expectation (with regard to $Q$) of the expected loss (\ref{eq:eps}). Broadly speaking, we average the score of the prediction $\bar{\vtheta}$ over the potential ground-truths $\vtheta \sim Q$. Aleatoric uncertainty is the expected loss that remains, even when the learner is perfectly informed about the ground-truth $\vtheta$ before predicting. Again, we average over the potential ground-truths $\vtheta \sim Q$. Epistemic uncertainty is the difference between the two, i.e., the expected loss reduction due to information about $\vtheta$. When $\ell$ is taken to be a strictly proper scoring rule, \eqref{eq:eunft} is also known as the Bregman information \citep{banerjeeOptimal2004}. We discuss three important loss-instantiations of the uncertainty measures: 

\begin{enumerate}[leftmargin=0.6cm]
\item[(1)] \emph{Log loss:} Instantiating the uncertainty measures \eqref{eq:tuau} with the log loss $\ell(\hvtheta,y) = -\log(\htheta_y)$ yields 
\begin{align*}
 \underbrace{\rm{S}(\bma)}_{\TU(Q)} &= \underbrace{\mathbb{E}_{\vtheta \sim Q}[\rm{S}(\vtheta)]}_{\AU(Q)} + \underbrace{\bE^{}_{\vtheta \sim Q}[\rm{KL}(\vtheta \parallel \bma)]}_{\EU(Q)},
\end{align*}
where $\rm{S}(\cdot)$ and $\rm{KL}(\cdot \parallel \cdot)$ denote the Shannon entropy and Kullback-Leibler divergence, respectively. Clearly, the log loss instantiations correspond to the entropy-based measures \citep{depeweg2018decomposition}. Although these measures have been criticized \citep{wimmer2023quantifying}, they remain the most commonly used uncertainty measures in the classification setting.
\item[(2)] \emph{Brier loss} (or quadratic loss): Similarly, fixing $\ell(\hvtheta,y) = \sum_{k=1}^K(\htheta_k - \llbracket k = y\rrbracket)^2$ yields 
\begin{align*}
\underbrace{1 \! - \!\! \sumK\bar{\theta}_k^2}_{\TU(Q)} = \underbrace{\mathbb{E}_{\vtheta \sim Q} \Big[ 1 \!- \!\! \sumK\theta^2_k \Big]}_{\AU(Q)} + \underbrace{\bE^{}_{\vtheta \sim Q}\sumK(\bar{\theta}_k - \theta_k)^2}_{\EU(Q)}
\end{align*}
The Brier loss \citep{brierVerificationOf1950} is another strictly proper scoring rule, which is often used as a measure of calibration \citep{mindererRevisitingCalibration2021, clarteExpectationConsistency2023}. The decomposition generates the measures of expected Gini impurity for aleatoric uncertainty. The Gini impurity quantifies the probability of misclassification when predicting randomly according to the ground-truth distribution, i.e. $\hvtheta = \vtheta$. The measure of epistemic uncertainty is the expected squared difference. This measure has also been proposed by \citet{smithUnderstandingMeasures2018}.
\item[(3)] \emph{Zero-one loss:} With  $\ell(\hvtheta,y) = 1-\llbracket \argmax_k \htheta_k = y\rrbracket$, we get the following instantiations:
\begin{align*}
\underbrace{1 - \max_k \bar{\theta}_k}_{\TU(Q)} &= \underbrace{\bE_{\vtheta \sim Q}[1 - \max_k\theta_k]}_{\AU(Q)} \\ &+ \underbrace{\bE_{\vtheta \sim Q}[\max_k \theta_k - \theta_{\argmax_k\bar{\theta}_k}]}_{\EU(Q)}. 
\end{align*}
The aleatoric component is the expected complement of the confidence. Assuming $\hvtheta = \vtheta$, it quantifies the probability of misclassification when predicting the class with the highest probability. Quantifying uncertainty on the basis of the confidence of the model is common \citep{hendrycksABaseline2017}, but in a second-order representation, this measure has not been used before. Interestingly, this component aligns with the measure of aleatoric uncertainty proposed and axiomatically justified by \citet{sale2024second}. The epistemic component of this decomposition has, to the best of our knowledge, not been used in machine learning so far. It is minimized when all first-order distributions $\vtheta$ in the support of the second-order distribution $Q$ have the same argmax as the Bayesian model average $\bma$, where the support is defined as $\operatorname{supp}(Q) = \{\vtheta \in \Delta_K : Q(\vtheta) > 0\}$. 
Minimizing this component guarantees that every first-order distribution in the support of $Q$ agrees on the class with highest probability (i.e., consensus on the most likely class), but it does not eliminate all epistemic uncertainty about the true conditional distribution $\vtheta$. In particular, variability in the assigned probability mass, such as how confident the predictors are about that top class or how much weight is placed on runner-up classes, can persist. Put differently, the zero-one loss epistemic measure only captures label-level disagreement about which class is most likely and is blind to finer-grained uncertainty in the shape of the distributions; even when it is zero, $Q$ may still be diffuse and reflect unresolved uncertainty about $\vtheta$ beyond the predicted label.
\end{enumerate}

\emph{Remark.} Let $\ell: \ksimplex \times \cY \rightarrow \overline{\mathbb{R}}$ be a strictly proper scoring rule with associated convex potential $G: \ksimplex \rightarrow \mathbb{R}$ such that $D_\ell(\hat{\vtheta},\vtheta)=L_\ell(\hat{\vtheta},\vtheta)-L_\ell(\vtheta,\vtheta)$
is the \emph{Bregman divergence} induced by $G$. For a second-order distribution $Q \in \ksimplextwo$, let $\bar{\vtheta} \defeq \mathbb{E}_{\vtheta\sim Q}[\vtheta]$ denote the (Bayesian) model average. Then, the measure of epistemic uncertainty \eqref{eq:eunft} is exactly a Jensen gap of the convex potential $G$, which measures how spread out the posterior over first-order predictions is:
\begin{align*}
\EU(Q) = \mathbb{E}_{\vtheta\sim Q} \left[G(\vtheta)\right]  -  G \left(\bar{\vtheta}\right).
\end{align*}
This is easy to see: By the convex-potential representation of strictly proper scoring rules \citep{gneiting2007strictly}, $D_\ell(\hat{\vtheta},\vtheta) = G(\vtheta)-G(\hat{\vtheta})-\langle\nabla G(\hat{\vtheta}),\,\vtheta-\hat{\vtheta}\rangle$. With the prediction $\hat{\vtheta}=\bar{\vtheta}$ and taking the expectation over $Q$,
\begin{align*}
    \EU(Q) &= \mathbb{E}[D_{\ell}(\vtheta,\bar{\vtheta})] \\
     &= \mathbb{E}[G(\vtheta)]-G(\bar{\vtheta})-\big\langle\nabla G(\bar{\vtheta}),\,\mathbb{E}[\vtheta-\bar{\vtheta}]\big\rangle \\ 
     &= \mathbb{E}[G(\vtheta)]-G(\bar{\vtheta}),
\end{align*}
since $\mathbb{E}[\vtheta-\bar{\vtheta}]=0$. Thus, $\EU(Q)=0$ iff $Q$ is a point mass at $\bar{\vtheta}$; more generally, $\EU$ is monotone in the convex order, if $Q’$ is more dispersed than $Q$ but has the same mean, then $\EU(Q’) \ge \EU(Q)$. This characterization is epistemically appealing: $\EU$ grows precisely with posterior dispersion about the unknown ground‐truth distribution and vanishes only when the learner’s belief collapses to a single $\vtheta$, i.e., when no epistemic uncertainty remains.

\begin{figure*}[t!]
  \centering
  \includegraphics[width=.95\linewidth]{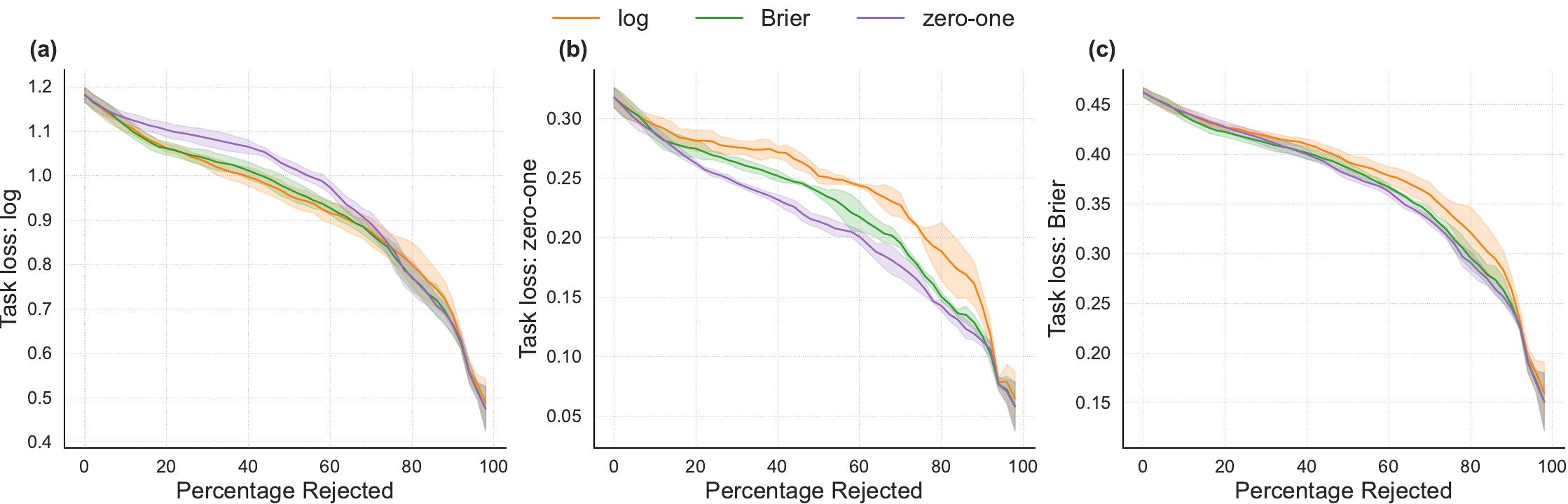}
\caption{Selective prediction with different task losses using the total uncertainty component as the rejection criterion where \textbf{(a)} uses the \emph{log loss} as task loss, \textbf{(b)} \emph{zero-one loss}, and \textbf{(c)} the \emph{Brier loss}, respectively. \emph{Results are averaged over three runs.}}
\label{fig:selective-prediction}
\end{figure*}

\section{Customized Uncertainty Quantification}\label{sec:custom}
Throughout, we write $\mathcal{L}(\Delta_K, \mathcal{Y})$ to denote the collection of all proper scoring rules $\ell:\Delta_K\times\cY\to\overline{\mathbb R}$. Moreover, let $U_{\ell}$ denote a mapping $\ksimplextwo \longrightarrow \mathbb{R}_{\geq 0}$, where $\ell \in \mathcal{L}(\Delta_K, \mathcal{Y})$. 

\emph{Task Loss vs.\ Uncertainty Loss.} In real‐world settings, we do not measure uncertainty purely for its own sake, but to understand to what extent it informs and improves performance on the downstream task.
In typical machine learning tasks, performance is evaluated on held‐out test data $\mathcal{D}_{\mathrm{test}}$ using a loss function $\ell_{\mathrm{task}}$, which we call the \emph{task loss}. To distinguish it from the scoring rule $\ell$ that parametrizes the uncertainty measure $U_{\ell}$, we refer to that rule as the \emph{uncertainty loss}. The task loss may have the same structure as the uncertainty loss, i.e., $\ell_{\rm{task}} : \Delta_K \times \mathcal{Y} \to \mathbb{R}_{\geq 0}$ quantifies the cost associated with predicting the probability distribution $\hat{\vtheta} \in \Delta_K$ when the true outcome is $y \in \mathcal{Y}$, and the overall loss is the average over the predictions on $\mathcal{D}_{\rm{test}}$. In general, however, $\ell_{\rm{task}}$ can be a complex loss function that is neither defined in an instance-wise manner nor decomposable over the data points in $\mathcal{D}_{\rm{test}}$. In selective prediction, for example, the performance is determined by the \emph{ordering} of the data points (according to their uncertainty). Thus, a loss cannot be assigned to an individual data point anymore. Instead, the uncertainty score assigned to a data point can only be assessed in comparison to others. We say a scoring rule $\ell$ is better aligned with the task loss $\ell_{\rm task}$ than another rule $\ell^{\prime}$ if using its induced uncertainty measure $U_{\ell}$ yields a strictly lower (expected) task loss than using $U_{\ell^{\prime}}$.

\begin{table*}[t!]
\centering
\begin{tabularx}{\textwidth}{XX>{\centering\arraybackslash}X>{\centering\arraybackslash}X>{\centering\arraybackslash}X}
\toprule
Dataset                      & Method & log loss & Brier loss & zero-one loss \\ \midrule
\multirow{3}{*}{\dataset{CIFAR-100}} & Dropout & $ \textbf{0.829} \scriptstyle{\pm 0.000} $ & $ 0.822 \scriptstyle{\pm 0.000} $ & $ 0.702 \scriptstyle{\pm 0.001} $\\
                            & Ensemble & $ \textbf{0.860} \scriptstyle{\pm 0.001} $ & $ 0.852 \scriptstyle{\pm 0.002} $ & $ 0.762 \scriptstyle{\pm 0.002} $ \\
                            & Laplace & $ \textbf{0.845} \scriptstyle{\pm 0.002} $ & $ 0.836 \scriptstyle{\pm 0.001} $ & $ 0.808 \scriptstyle{\pm 0.001} $ \\\midrule
\multirow{3}{*}{\dataset{Places365}}  & Dropout     & $ \textbf{0.837} \scriptstyle{\pm 0.001} $ & $ 0.828 \scriptstyle{\pm 0.001} $ & $ 0.714 \scriptstyle{\pm 0.002} $ \\
                            & Ensemble & $ \textbf{0.856} \scriptstyle{\pm 0.002} $ & $ 0.846 \scriptstyle{\pm 0.002} $ & $ 0.758 \scriptstyle{\pm 0.005} $ \\
                            & Laplace & $ \textbf{0.863} \scriptstyle{\pm 0.003} $ & $ 0.850 \scriptstyle{\pm 0.003} $ & $ 0.825 \scriptstyle{\pm 0.004} $ \\\midrule
\multirow{3}{*}{\dataset{SVHN}}       & Dropout     & $ \textbf{0.835} \scriptstyle{\pm 0.000} $ & $ 0.830 \scriptstyle{\pm 0.000} $ & $ 0.701 \scriptstyle{\pm 0.002} $ \\
                            & Ensemble & $ \textbf{0.872} \scriptstyle{\pm 0.005} $ & $ 0.868 \scriptstyle{\pm 0.005} $ & $ 0.776 \scriptstyle{\pm 0.007} $ \\
                            & Laplace & $ \textbf{0.865} \scriptstyle{\pm 0.005} $ & $ 0.856 \scriptstyle{\pm 0.004} $ & $ 0.826 \scriptstyle{\pm 0.006} $ \\ \bottomrule 
\end{tabularx}
\caption{OoD detection with \dataset{CIFAR-10} as in-Distribution data based on epistemic uncertainty. \emph{The mean and standard deviation of the AUROC over three runs are reported.} Best results are highlighted in \textbf{bold}.} 
\label{tab:ood}
\end{table*}

\subsection{Selective Prediction}\label{subsec:selective}
Selective prediction is a task where the model can abstain from making a prediction on some inputs if it is uncertain about the correct outcome. 
Usually, performance on this task is measured using a hold-out dataset, for example, a test set.
Formally, let the test set be denoted as $\cD_{\rm{test}} = \{(\vx_i, y_i)\}_{i = 1}^n$, where for each instance $\vec{x}_i$, a predictive model outputs a second-order distribution  $Q_i \in \ksimplextwo$.
Further, for $\alpha \in [0,1]$ let $k =  \floor{\alpha n}$ be a (fixed) rejection level, which dictates the number of instances for which the model is allowed to abstain from making a prediction. 
The permutation $\pi$ of $\{1,\dots,n\}$ is defined so that
\begin{align*}
    U_{\ell}(Q_{\pi(1)}) \geq  U_{\ell}(Q_{\pi(2)}) \geq \dots \geq U_{\ell}(Q_{\pi(n)}).
\end{align*}
In other words, the permutation $\pi$ sorts instances by their uncertainty, as quantified by the measure $U_{\ell}$. Again, let $\ell \in \mathcal{L}(\Delta_K, \mathcal{Y})$ be a loss function. Then, the area under the loss-rejection curve (AULC) is defined as  
\begin{align}\label{def:aulc}
    \rm{AULC} = \int_{0}^{1} \left( \frac{1}{\floor{\alpha n}} \sum_{j = 1}^{\floor{\alpha n}} \ell^*(\hat{\vtheta}_{\pi(j)}, y_{\pi(j)}) \right) \, d\alpha.
\end{align}
Taking the expectation (over the randomness in the labels) in \eqref{def:aulc} yields the \emph{expected} AULC.
In the context of selective prediction, AULC can be interpreted as the task loss $\ell_{\rm{task}}$, as it quantifies the expected prediction error over varying levels of instance rejection. We call $\ell^{\star}$ in \eqref{def:aulc} \emph{auxiliary} task loss.

\begin{proposition}\label{prop:arc-tu}
Let $\hat{\vtheta} \in \Delta_K$ be a (first-order) prediction and $\ell \in \mathcal{L}(\Delta_K, \mathcal{Y})$. Then, the expected AULC is minimized by ordering test instances in non-decreasing order of their (instance-wise) expected loss $\mathbb{E}_{y \sim \theta}\bigl[\ell(\hat{\vtheta},y)\bigr]$.
\end{proposition}

\begin{proof}
    For $\alpha \in [0,1]$, the area under the loss-rejection curve (AULC) is defined as  
\begin{align*}
    \rm{AULC} = \int_{0}^{1} \left( \frac{1}{\floor{\alpha n}} \sum_{j = 1}^{\floor{\alpha n}} \ell(\hat{\theta}_{\pi(j)}, y_{\pi(j)}) \right) \, d\alpha.
\end{align*}
Define $c_{\pi(j)} \defeq \mathbb{E}\bigl[\ell(\hat{\theta}_{\pi(j)},y_{\pi(j)})\bigr] $. Then, the \emph{expected} area under the loss-rejection curve is given by
\begin{align}
\mathbb{E}[\mathrm{AULC}]
&=\int_{0}^{1} \left( \frac{1}{\lfloor \alpha n \rfloor} \sum_{j=1}^{\lfloor \alpha n \rfloor} \mathbb{E}\bigl[\ell(\hat{\theta}_{\pi(j)},y_{\pi(j)})\bigr] \right) d\alpha \nonumber \\&=\int_{0}^{1} \left( \frac{1}{\lfloor \alpha n \rfloor} \sum_{j=1}^{\lfloor \alpha n \rfloor} c_{\pi(j)} \right) d\alpha.
\end{align}
Consequently, we can approximate the integral by a Riemann sum with step $\Delta \alpha=\frac{1}{n}$:
\begin{align*}
\mathbb{E}[\mathrm{AULC}]
\approx \frac{1}{n} \underbrace{\sum_{k=1}^{n}\left( \frac{1}{k} \sum_{j=1}^{k} c_{\pi(j)} \right)}_{\eqqcolon S({\pi}) }.
\end{align*}
Then, interchanging the order of summation yields
\begin{align*}
S(\pi) =\sum_{k=1}^{n}\sum_{j=1}^{k}\frac{1}{k}\,c_{\pi(j)} =\sum_{j=1}^{n}c_{\pi(j)}\sum_{k=j}^{n}\frac{1}{k}.
\end{align*}
With weights $w_j=\sum_{k=j}^{n}\frac{1}{k}$ we finally get 
$ S(\pi)=\sum_{j=1}^{n}w_j\,c_{\pi(j)}$. Since $w_1\ge w_2\ge\cdots\ge w_n>0$, the rearrangement inequality implies that the sum $\sum_{j=1}^{n}w_j\,c_{\pi(j)}$ is minimized when $ c_{\pi(1)}\le c_{\pi(2)}\le\cdots\le c_{\pi(n)}$. 
\end{proof}

Now, if $\hat{\vtheta} = \bar{\vtheta}$, then considering the expectation (with respect to the learner's belief $Q$) over $\mathbb{E}_{y \sim \vtheta}\bigl[\ell(\bar{\vtheta},y)\bigr]$ yields the measure of total uncertainty in \eqref{eq:tuau}. This leads to an important observation: In selective prediction, when determining the ordering of test instances (e.g., based on uncertainty measures), the most sensible strategy to minimize the expected AULC, as established in \cref{prop:arc-tu}, is to order them according to the (predicted) \emph{total} uncertainty with uncertainty loss $\ell$ given by $\ell^*$ in (\ref{def:aulc}).

As an aside, let us note that loss-rejection curves (or, analogously, accuracy-rejection curves) are commonly used as a means to evaluate aleatoric and epistemic uncertainty measures, too, which means the curves are constructed for these measures as selection criteria \citep{eyke_new, saleLabelWise2024}. In light of our finding that total uncertainty is actually the right criterion, this practice may appear somewhat questionable, as it means that aleatoric and epistemic uncertainty measures are evaluated on a task they are actually not tailored to. 

\emph{Empirical Results.} We generate loss-rejection-curves by rejecting the predictions for instances on which the predictor is most uncertain and computing the loss on the remaining subset \citep{huhn2008fr3}. Given a good uncertainty quantification method, the loss should monotonically decrease with the percentage of rejected instances, because the model misclassifies instances with low uncertainty less often than instances with high uncertainty. 

In particular, we train a \texttt{RandomForest} classifier, an ensemble of decision trees, on the \dataset{CoverType} dataset \citep{blackardCoverType1998}. In \cref{fig:selective-prediction} we show the results for three different task losses using total uncertainty as the rejection criterion. This validates the theory: the rejection ordering is optimal when the uncertainty loss is aligned with the task loss. The effect is most pronounced with zero-one loss, the canonical loss for selective prediction \citep{nadeemARCs2009, geifmanSelectiveClassification2017}. Experimental details and additional experiments with varying uncertainty components, models, and datasets are deferred to the supplementary material. 
The code is at: \url{https://github.com/pwhofman/task-specific-uncertainty}.

Finally, two practical caveats are worth noting. The optimal ordering in \cref{prop:arc-tu} presumes that the uncertainty values faithfully reflect the instance-wise expected loss; poor posterior approximations will blur that ranking and weaken the gains, so investing in better second-order beliefs (e.g., via ensembling etc.) improves selective prediction in practice. Conversely, using only aleatoric or only epistemic components as selection criterion, rather than total uncertainty aligned with the task loss, can be misleading.

\subsection{Out-of-Distribution Detection}

\begin{figure*}[t!]
  \centering
  \includegraphics[width=.95\linewidth]{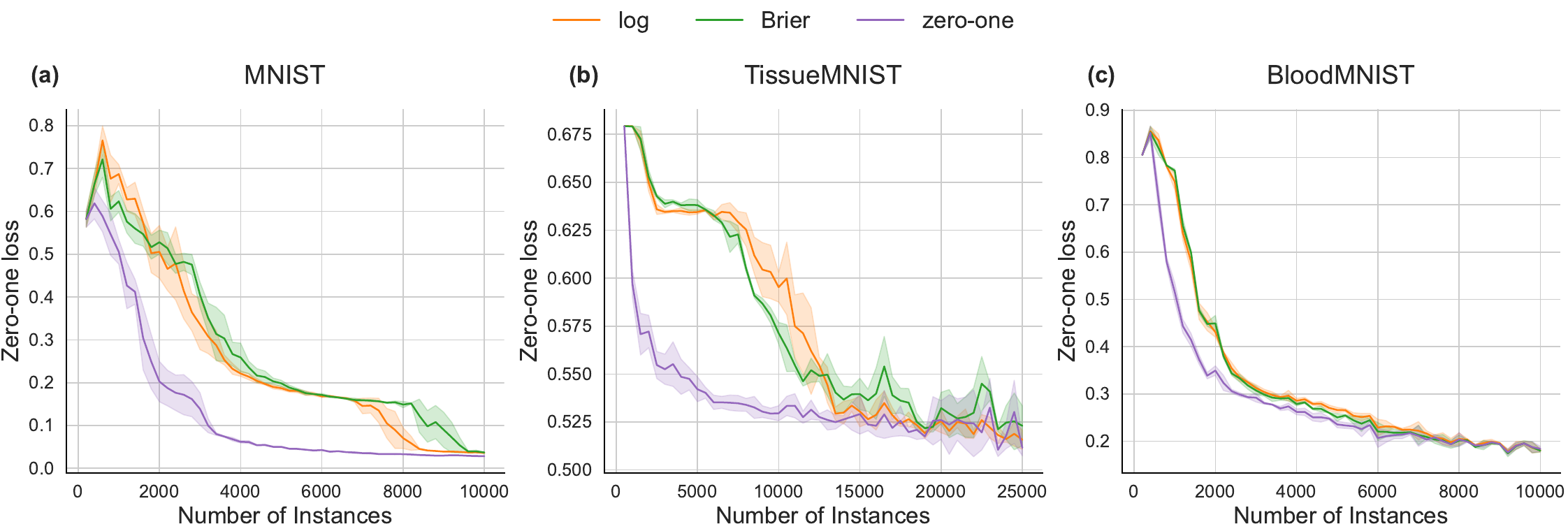}
\caption{Active learning with different datasets using the epistemic uncertainty component to query new instances, where \textbf{(a)} is based on the \dataset{MNIST} dataset, \textbf{(b)} \dataset{TissueMNIST}, and \textbf{(c)} \dataset{BloodMNIST}, respectively. \emph{Results are averaged over three runs.}}
\label{fig:active-learning}
\end{figure*}

Another complementary downstream task to assess and compare the quality of uncertainty quantification is out-of-distribution (OoD) detection.
We first train the model on in-distribution (iD) data and evaluate its uncertainty on held-out iD test instances. Then, we present the model with OoD samples and compute their uncertainties as well. Because the model has not seen the OoD domain during training, it should register \emph{higher} epistemic uncertainty on those inputs. Being able to separate iD from OoD points is essential for reliability, since predictions outside the training distribution are inherently less trustworthy.

\emph{Empirical Results.} We train a \texttt{ResNet18} \citep{heDeepResidual2016} on \dataset{CIFAR-10} and approximate the second-order predictive distribution using three methods: deep ensembles \citep{lakshminarayananDeep2017}, Monte Carlo Dropout \citep{galDropout2016}, and a Laplace approximation around the trained parameters \citep{daxbergerLaplace2021}. Table \ref{tab:ood} reports epistemic uncertainty performance on OoD datasets with \dataset{CIFAR-10} as the in-distribution data, comparing three loss-based instantiations of the epistemic uncertainty measure \eqref{eq:eunft}. The results show that epistemic uncertainty measures instantiated with the log loss (i.e., mutual information) achieves the best OoD performance, which may be explained by the log loss penalizing over-confident predictions, thus improving the separation of iD and OoD samples.
This aligns with, and helps justify, its widespread use in second-order uncertainty representations for OoD detection \citep{muscanyiBenchmarkingUncertainties2024}. On the contrary, the zero-one loss ignores the confidence magnitude. Epistemic uncertainty is zero when the argmaxes ``agree'', yielding a weaker separation signal.
Additional results using \dataset{ImageNet} \citep{dengImageNet2009} and \dataset{Food101} \citep{bossardFood2014} as iD data are reported in the supplementary material; they further underscore the superior performance of the log-loss–based instantiation of \eqref{eq:eunft}. 

We note that OoD detection is qualitatively different from selective prediction or active learning because the “task’’ itself is shaped by how the OoD examples are constructed. Covariate shifts, semantic shifts, near- versus far-OoD, and synthetic versus natural perturbations all induce different separability structures between in- and out-of-distribution data, so any ranking-based metric, such as the AUROC used here, conflates properties of the uncertainty measure with the particular flavor of shift being evaluated. That mutual information (the log loss instantiated epistemic uncertainty measure) performs best in our experiments indicates it is especially sensitive to the kinds of unfamiliarity present in these benchmarks; however, this performance should not be interpreted as universally dominant without qualification. For example, \citet{li2025out} critically re-examine common OoD detection pipelines and argue that many of them are effectively asking the wrong questions, conflating surrogate signals with the underlying notion of ``out-of-distributionness''.

This finding, too, echoes the paper’s central message: the “best” uncertainty measure depends on the downstream task (there is \emph{no} one size fits all).  In selective prediction, aligning the uncertainty loss with the task loss yields optimal behavior, whereas in the OoD detection benchmarks, the log loss instantiation of epistemic uncertainty (mutual information) empirically separates familiar from unfamiliar inputs most reliably, albeit with the caveat that its dominance can depend on the specific distribution shift and evaluation setup.

\subsection{Active Learning}

Active learning is another popular downstream task that is frequently used to evaluate uncertainties, since its success hinges on identifying examples about which the model is most uncertain in an epistemic sense and thus will benefit most from labeling \citep{mpub443}. Its objective is to reach strong performance with as few labels as possible. Beginning from a (small) labeled seed set, the learner iteratively selects unlabeled examples to be annotated by an oracle. Many query strategies leverage epistemic uncertainty \citep{nguyenEpistemic2019, kirschBatch2019, margrafALPBench2024}. The source of (epistemic) uncertainty we care about in this setting is label disagreement: whether plausible predictive distributions disagree on the most likely class. Accordingly, the uncertainty measure should directly reflect that disagreement. Many common measures instead conflate label-level ambiguity with other sources of epistemic uncertainty, for example, uncertainty about the full first-order distribution, even when all plausible predictors agree on the top label. That extra sensitivity can cause less effective queries, whereas the zero-one loss instantiation isolates true disagreement on the predicted label. Here, we run pool-based active learning: in each round, we score candidates using different instantiations of the epistemic uncertainty measure \eqref{eq:eunft} and query the highest-uncertainty examples.

\emph{Empirical Results.} We use \dataset{MNIST} \citep{leCunGradientBased1998}, \dataset{FashionMNIST} \citep{xiao2017fashion}, and multiclass subsets of the \dataset{MedMNIST} collection \citep{yangMedmnistv22023}. The benchmark includes both color and grayscale tasks; for color inputs we employ a small convolutional architecture based on the \texttt{LeNet} \citep{leCunGradientBased1998} architecture, and for grayscale data a small fully connected network. The second-order predictive distribution is approximated via Monte Carlo Dropout, as is standard in image-based active learning \citep{galDeepBayesian2017, kirschBatch2019}. \cref{fig:active-learning} shows the task loss (zero-one) versus the number of labeled examples. On all shown datasets, epistemic uncertainty sampling using the zero-one–loss instantiation delivers the best label efficiency. Also for active learning experiments, further datasets and ablations can be found in the supplementary material.

The zero-one loss instantiation performs well because it targets disagreement over the most likely label among plausible predictive distributions, i.e., the kind of label disagreement that active learning specifically seeks to resolve. The most informative unlabeled examples are those where the model is unsure about the correct label because different plausible (first-order) distributions disagree on which class is most likely; examples where all of them agree on a single top class add little new information. The zero-one instantiation of the epistemic uncertainty measure captures this disagreement: its value is zero when there is consensus on the predicted class and grows only when there is genuine uncertainty about which label should be chosen. In contrast, other instantiations (e.g., log or Brier) may signal uncertainty even though the predicted label would remain unchanged, since they respond to finer variations in the second-order distribution that do not affect the top choice, potentially resulting in less focused queries. Once more, we see that one size does \emph{not} fit all, active learning benefits most from the specific form of uncertainty captured by zero-one loss (disagreement on the predicted label) rather than a broad, undifferentiated uncertainty measure.

\subsection{Findings and Insights}
Taken together, our theoretical and empirical results form a coherent prescription: uncertainty quantification must be \emph{customized} to the downstream task, not treated as a one-size-fits-all solution. 
On selective prediction, aligning the uncertainty loss with the task loss yields the optimal rejection ordering; for OoD detection, the log loss instantiation (mutual information) best isolates unfamiliar inputs; and for active learning, resolving label-level disagreement via the zero-one epistemic measure drives the most label-efficient gains. Because these three tasks are among the most widely used evaluation paradigms in both research and practice, this paper serves as a reality check: without careful alignment, empirical comparisons can be misleading.

\section{Related Work}
\emph{Uncertainty Quantification.}
For second-order distributions, the most commonly used measures are based on information-theoretic decompositions of Shannon entropy \citep{depeweg2018decomposition}. 
However, these measures have been criticized for violating several properties that uncertainty measures should fulfill \citep{wimmer2023quantifying}. 
A generalization has been proposed, which considers different instantiations of the predicting model and approximations of the predictive first-order distribution \citep{schweighofer2025on}. 
Beyond information-theoretic approaches, there have also been proposals of uncertainty measures that are based on a decomposition of risk or loss. \citet{lahlouDirect2023} propose a method to directly quantify epistemic uncertainty based on the difference between total risk and Bayes risk. \citet{gruberUncertainty2023} use a general bias-variance decomposition to quantify uncertainty based on the Bregman information \citep{banerjeeOptimal2004} with mutual information as a concrete instantiation. Recently, \citet{hofman2024quantifying, hofman2024CredalApproach, schweighofer2025on, kotelevskii2025risk} have introduced an uncertainty quantification framework based on proper scoring rules. 

\emph{Connection to Downstream Tasks.}
A complimentary line of work starts from the downstream predictive task. 
\citet{smithRethinkingAleatoric2024} argue that one should reason about uncertainty by first specifying the predictive task at hand.
Other related work compares alternative representations of uncertainty and shows that the chosen representation can substantially affect task performance \citep{muscanyiBenchmarkingUncertainties2024, deJongHowDisentangled2024}. Likewise, several works observe that different uncertainty measures exhibit markedly different performance profiles across downstream tasks \citep{schweighofer2025on, kotelevskii2025risk}.
To the best of our knowledge, our contribution is the first to explicitly connect, in a single framework, downstream \emph{tasks} to specific uncertainty \emph{measures}, theoretically (via loss-alignment results) and empirically (across selective prediction, OoD detection, and active learning).

\section{Concluding Remarks}
We have argued and shown that uncertainty quantification is not a one-size-fits-all endeavor: the usefulness of a given uncertainty measure depends critically on the downstream task and how that task evaluates performance. On the theoretical side, we tied the construction of uncertainty measures to proper scoring rules and proved that optimal instance ordering in selective prediction arises when the uncertainty loss is aligned with the task loss. Empirically, we confirmed this principle across three canonical evaluation paradigms. For selective prediction, total uncertainty instantiated with the task-aligned loss yields the best rejection behavior; for out-of-distribution detection, the log loss based epistemic measure (mutual information) most reliably identifies unfamiliar inputs; and for active learning, querying based on the zero-one loss epistemic disagreement delivers the strongest label efficiency.
Beyond these specific findings, the broader implication is practical: researchers and practitioners  should choose and report uncertainty measures with their target objectives in mind. Blindly applying generic uncertainty scores or mixing selection and evaluation criteria can obscure real performance differences and lead to misleading conclusions.

\emph{Limitations and Future Work.} All of our downstream benefits hinge on reasonably faithful second-order beliefs; poor posterior approximations can degrade the expected gains, underscoring the value of improved uncertainty representations (e.g., better ensembles, or more expressive posterior approximations). It remains of great interest for both the machine learning community and practitioners to understand how to (empirically) \emph{evaluate} uncertainty itself, an inherently difficult problem, since true uncertainty is unobserved and must be judged indirectly via downstream objectives. Our results \emph{caution} against one-size-fits-all practices and instead point toward task-aligned evaluation protocols and benchmarks, alongside extending the framework beyond multiclass classification to regression, structured outputs, and cost-sensitive or imbalanced settings.

\section{Acknowledgments}
Yusuf Sale is supported by the DAAD program Konrad Zuse Schools of Excellence in Artificial Intelligence, sponsored by the Federal Ministry of Education and Research.


{\small\bibliography{references}}




\appendix
\onecolumn
\setcounter{secnumdepth}{2}
\renewcommand\thesection{\Alph{section}}
\renewcommand\thesubsection{\thesection.\arabic{subsection}}

\section{Derivations}\label{app:derivations}
In this section, we derive measures of total, aleatoric, and epistemic uncertainty, instantiated for the log, Brier, and zero-one losses, respectively. We begin with the following definitions:
\begin{align}
    \TU(Q) &= \bE_{\vtheta \sim Q}\big[L_\ell(\bma, \vtheta)\big] = \bE_{\vtheta \sim Q}\Big[\bE_{y \sim \vtheta}\big[\ell(\bma, y)\big]\Big] \label{eq:tu-app}, \\[0.2cm]
    \AU(Q) &= \bE_{\vtheta \sim Q}\big[H_\ell(\vtheta)\big] = \bE_{\vtheta \sim Q}\big[L_\ell(\vtheta, \vtheta)\big] = \bE_{\vtheta \sim Q}\Big[\bE_{y \sim \vtheta}\big[\ell(\vtheta, y)\big]\Big] \label{eq:au-app}, \\[0.2cm]
    \EU(Q) &= \bE_{\vtheta \sim Q}\big[D_\ell(\bma, \vtheta)\big] = \bE_{\vtheta \sim Q}\big[L_\ell(\bma, \vtheta) - L_\ell(\vtheta, \vtheta)\big] = \bE_{\vtheta \sim Q}\Big[\bE_{y \sim \vtheta}\big[\ell(\bma, y) - \ell(\vtheta, y)\big]\Big].
    \label{eq:eu-app}
\end{align}
Given a second-order distribution $Q \in \ksimplextwo$, we predict with the Bayesian model average (BMA), i.e., $\hat{\vtheta} \coloneqq \bma = \mathbb{E}_{\vtheta\sim Q}[\vtheta]$.

\subsection{Log loss}
With $\ell(\bma, y) = -\log(\btheta_y)$,
\begin{align*}
    \TU(Q) &= \bE_{\vtheta \sim Q}\Big[\bE_{y \sim \vtheta}\big[-\log(\btheta_y)\big]\Big] \\
    &= \bE_{\vtheta \sim Q}\Big[\sumK \theta_k\big(-\log(\btheta_k)\big)\Big] \\
    &= \sumK \bar{\theta}_k\big(-\log(\bar{\theta}_k)\big) \\
    &= \operatorname{S}(\bma), \\ \\
    \AU(Q) &= \bE_{\vtheta \sim Q}\big[\bE_{y \sim \vtheta}[-\log(\theta_y)]\big] \\
    &= \bE_{\vtheta \sim Q}\big[\operatorname{S}(\vtheta)\big], \\ \\
    \EU(Q) &= \bE_{\vtheta \sim Q}\Big[\bE_{y \sim \vtheta}\big[-\log(\btheta_y) + \log(\theta_y)\big]\Big] \\
    &= \bE_{\vtheta \sim Q}\big[\KL(\vtheta \parallel \bma)\big].
\end{align*}

\subsection{Brier loss}
With $\ell(\bma,y) = \sumK\big(\bar{\theta}_k - \llbracket k = y\rrbracket\big)^2 = - 2 \bar{\theta}_y + \sumK \bar{\theta}_k^2 + 1$, 
\begin{align*}
    \TU(Q) &= \bE_{\vtheta \sim Q}\Big[\bE_{y \sim \vtheta}\big[-2\btheta_y + \sumK\btheta_k^2 + 1\big]\Big] \\
    &= \bE_{\vtheta \sim Q}\Big[\sumK\theta_k(-2\btheta_k)\Big] + \sumK\btheta_k^2 + 1 \\
    &= \sumK\big(\btheta_k(-2\btheta_k) + \btheta_k^2\big) + 1 \\
    &= \sumK(-2\btheta_k^2 + \btheta_k^2) + 1 \\
    &= 1 - \sumK\btheta_k^2, \\ \\
    \AU(Q) &= \bE_{\vtheta \sim Q}\Big[\bE_{y \sim \vtheta}\big[-2\theta_y + \sumK\theta_k^2 + 1\big]\Big] \\
    &= \bE_{\vtheta \sim Q}\Big[\sumK\big(\theta_k(-2\theta_k) + \theta_k^2\big) + 1\Big] \\
    &= \bE_{\vtheta \sim Q}\Big[\sumK(-2\theta_k^2 + \theta_k^2) + 1\Big] \\
    &= \bE_{\vtheta \sim Q}\Big[1 - \sumK\theta_k^2\Big], \\ \\
    \EU(Q) &= \bE_{\vtheta \sim Q}\Big[\bE_{y \sim \vtheta}\big[-2\btheta_y + \sumK\btheta_k^2 + 1 + 2\theta_y - \sumK\theta_k^2 - 1\big]\Big] \\
    &= \bE_{\vtheta \sim Q}\Big[\sumK \big(\theta_k(-2\btheta_k + 2\theta_k) + \btheta_k^2 - \theta_k^2\big)\Big] \\
    &= \bE_{\vtheta \sim Q}\Big[\sumK \big(-2\btheta_k\theta_k + 2\theta_k^2 + \btheta_k^2 - \theta_k^2\big)\Big] \\
    &= \bE_{\vtheta \sim Q}\Big[\sumK \big(\btheta_k^2 -2\btheta_k\theta_k + \theta_k^2\big)\Big] \\
    &= \bE_{\vtheta \sim Q}\Big[\sumK \big(\btheta_k - \theta_k\big)^2\Big]. \\
\end{align*}

\subsection{Zero-one loss}
With $\ell(\bma,y) = 1-\llbracket \argmax_k \btheta_k = y\rrbracket$, 
\begin{align*}
    \TU(Q) &= \bE_{\vtheta \sim Q}\Big[\bE_{y \sim \vtheta}\big[1 - \llbracket \argmax_k \btheta_k = y\rrbracket\big]\Big] \\
    &= \bE_{\vtheta \sim Q}\Big[\sumKK \theta_{k'}\big(1 - \llbracket \argmax_k \btheta_k = k'\rrbracket\big)\Big] \\
    &= \sumKK \bar{\theta}_{k'}\big(1 - \llbracket \argmax_k \btheta_k = k'\rrbracket\big) \\
    &= 1 - \max_k \btheta_k, \\ \\
    \AU(Q) &= \bE_{\vtheta \sim Q}\Big[\bE_{y \sim \vtheta}\big[1 - \llbracket \argmax_k \theta_k = y\rrbracket\big]\Big] \\
    &= \bE_{\vtheta \sim Q}[1 - \max_k \theta_k], \\ \\
    \EU(Q) &= \bE_{\vtheta \sim Q}\Big[\bE_{y \sim \vtheta}\big[1 - \llbracket \argmax_k \btheta_k = y\rrbracket - 1 + \llbracket \argmax_k \theta_k = y\rrbracket\big]\Big] \\
    &= \bE_{\vtheta \sim Q}\Big[\bE_{y \sim \vtheta}\big[\llbracket \argmax_k \theta_k = y\rrbracket - \llbracket \argmax_k \btheta_k = y\rrbracket\big]\Big] \\
    & = \bE_{\vtheta \sim Q}[\max_k \theta_k - \theta_{\underset{k}{\argmax}\  \btheta_k}].
\end{align*}

\section{Experimental Details}\label{app:details}
In the following, we provide additional details regarding the experimental setup. We split this into descriptions of the models and training procedure, the uncertainty methods and their application, and the downstream tasks.
The code is written in Python 3.10.12 and relies heavily on PyTorch \citep{pytorch2019}.

\subsection{Compute Resources}
The experiments were conducted using the compute resources summarized in \cref{tab:compute}. The reported results required approximately 50 GPU hours and an additional 10 CPU hours.

\begin{table}[ht] 
\centering
\begin{tabularx}{\textwidth}{XX}
\toprule
Resource & Details \\ \midrule
GPU      & 2x NVIDIA A40 48GB GDDR \\
CPU      & AMD EPYC MILAN 7413 24 Cores / 48 Threads \\
RAM      & 128GB DDR4-3200MHz ECC DIMM \\
Storage  & 2x 480GB Samsung Datacenter SSD PM893\\
\bottomrule
\end{tabularx}
\caption{Compute resources.}
\label{tab:compute}
\end{table}

\subsection{Datasets}
\begin{table}[ht]
\centering
\begin{tabularx}{\textwidth}{lXX}
\toprule
Dataset    & Reference & License \\\midrule
\dataset{CoverType}  & \citep{blackardCoverType1998} & CC BY. \\
\dataset{Poker Hand} & \citep{cattralPokerHand2006} & CC BY. \\ 
\dataset{CIFAR-10}   & \cite{krizhevsky2009learning} & Unknown. \\
\dataset{CIFAR-100}  & \citep{krizhevsky2009learning}    & Unknown.        \\
\dataset{Places365}     & \citep{zhouPlaces2018}          & CC BY. \\
\dataset{SVHN}       & \citep{netzerReading2011}          & Non-commercial use.         \\
\dataset{ImageNet}   & \citep{dengImageNet2009}          & Non-commercial research/educational use.        \\
\dataset{ImageNet-O} & \citep{hendrycksNatural2021}          & MIT License.        \\
\dataset{Food101}    & \citep{bossardFood2014}          & Unknown.        \\
\dataset{MNIST} & \citep{leCunGradientBased1998} & CC BY. \\
\dataset{FashionMNIST} & \citep{xiao2017fashion} & MIT License. \\
\dataset{MedMNIST} & \citep{yangMedmnistv22023} & CC BY. \\ \bottomrule
\end{tabularx}
\caption{Datasets with references and licenses.}
\label{tab:datasets}
\end{table}

Table~\ref{tab:datasets} lists all datasets used in our experiments. We use the dedicated train-test split for all datasets when available. For the \dataset{CoverType} and \dataset{Poker Hand} datasets, which do not provide such a split, we randomly partition the data into a $70\%-30\%$ train–test split.
During training on \dataset{CIFAR-10}, we normalize images using the per-channel mean and standard deviation of the training set and apply random cropping and horizontal flipping. For out-of-distribution detection all iD and OoD instances undergo the same transformations.

\subsection{Models}
We train the following models.

\paragraph{Random Forest.}
The \texttt{RandomForest} is fit on \dataset{CoverType} using 20 trees and a maximum depth of 5 and on \dataset{Poker Hand} with 20 trees and a maximum depth of 20. For the remaining hyper-parameters, we use the defaults provided by sklearn \citep{scikit-learn}. 

\paragraph{Multilayer Perceptron.}
The Multilayer Perceptron (MLP) consists of an input layer with 784 features with ReLU activations followed by a hidden layer with 100 features and ReLU activation, and an output layer whose size corresponds to the number of classes after which a softmax function is applied.

\paragraph{Convolutional Neural Network.} The Convolutional Neural Network (CNN) is based on the \texttt{LeNet5} architecture \citep{leCunGradientBased1998}. It takes a three-channel input and consists of two convolutional layers followed by two fully-connected layers. The first convolutional layer has 32 filters of size 5x5 and the second has 64 filters of size 5x5. Both layers are followed by a 2x2 max-pooling operation and a ReLU activation. The feature maps are flattened and passed to the first fully-connected layer with 800 features and ReLU activation. Finally, the size of the output layer depends on the number of classes and applies a softmax function.

The training hyper-parameters for the MLP and CNN depend on the dataset and are listed in \cref{tab:alconfig}.

\paragraph{ResNet.} We use the \texttt{ResNet18} \citep{heDeepResidual2016} implementation from Github\footnote{\url{https://github.com/kuangliu/pytorch-cifar}} for the \dataset{CIFAR-10} experiments. The ResNets are trained for 100 epochs using stochastic gradient descent with a learning rate of 0.001, weight decay at 5e-4 and momentum at 0.9. The cosine annealing learning rate scheduler \citep{loshchilovStochasticGradient2017} is used.

We use the following pre-trained models.
\paragraph{EfficientNet.} For the \dataset{ImageNet} experiments, we use the \texttt{EfficientNetV2S} implementation from PyTorch which was pre-trained on \dataset{ImageNet}.

\paragraph{VisionTransformer.} We use a \texttt{VisionTransformer} for the \dataset{Food101} experiments. It was pre-trained on \dataset{Imagenet21K} \citep{ridnikImagenet21K2021} and fine-tuned on \dataset{Food101}. It was downloaded from Hugging Face\footnote{\url{https://huggingface.co/nateraw/vit-base-food101}}.

\subsection{Uncertainty Representations}
We employ the following methods to enable the models to represent their uncertainty using second-order distributions or suitable approximations thereof.
\paragraph{Dropout.}
A Dropout layer before the final layer using the default probability $0.5$ from PyTorch is used. The other Dropout layers are turned off during evaluation.
\paragraph{Laplace Approximation.}
We use the Laplace package \citep{daxbergerLaplace2021} to fit the Laplace approximation on the last layer using the Kronecker-factored approximate curvature approximation, which are the default settings for this package. We obtain samples of the posterior by Monte Carlo sampling.
\paragraph{Deep Ensembles.} The deep ensemble is constructed by training 5 similar neural networks, relying on the randomness of the initialization and stochastic gradient descent to get diverse predictions \citep{lakshminarayananDeep2017}. 

\subsection{Downstream Tasks}
All tasks are run three times and for each run a new model is trained (only exception being pre-trained Dropout and Laplace models) and a random subset of the test data is sampled. For the downstream tasks, given an instance, we sample 20 conditional distributions in order to compute the total, aleatoric, and epistemic uncertainty. 
\paragraph{Selective Prediction.}
The selective prediction tasks are done using 10,000 instances that are randomly sampled from the dedicated test split of the datasets. 
\paragraph{Out-of-Distribution Detection.}
For the out-of-distribution detection task we also sample 10,000 instances from the test sets of the respective datasets, if possible. \dataset{ImageNet-O} has only 2,000 instances, thus we use all these instances and also sample 2,000 instances from the in-Distribution dataset. 
\paragraph{Active Learning.}
Active learning is performed by starting with training on only the initial instances. The model then uses epistemic uncertainty to sample new instances with a certain query budget. After this, the model is trained again and the performance on the dedicated test set of the respective dataset is evaluated. This process is repeated for 50 iterations. The settings used for the different datasets are shown in \cref{tab:alconfig}.  

\begin{table}[ht] 
\centering
\begin{tabularx}{\textwidth}{X>{\centering\arraybackslash}X>{\centering\arraybackslash}X>{\centering\arraybackslash}X>{\centering\arraybackslash}X}
\toprule
Dataset & Initial Instances & Query Budget & Learning Rate & Epochs \\ \midrule
\dataset{MNIST}   & 100               & 100          & 0.01          & 50     \\
\dataset{TissueMNIST} & 500 & 500 & 0.01 & 50 \\
\dataset{BloodMNIST} & 200 & 200 & 0.001 & 100 \\
\dataset{FashionMNIST}  & 500               & 500          & 0.01          & 50     \\ 
\dataset{OrganCMNIST} & 200 & 200 & 0.01 & 50 \\
\dataset{PathMNIST} & 1000 & 1000 & 0.001 & 100 \\
\bottomrule
\end{tabularx}
\caption{Parameters for the active learning experiments.}
\label{tab:alconfig}
\end{table}

\section{Additional Results}\label{app:results}
\subsection{Selective Prediction}\label{app:selective-prediction}
We present the selective prediction results for the \dataset{CoverType} dataset with aleatoric and epistemic uncertainty as the rejection criteria in 
\cref{fig:app-sp-au-eu-covtype}, confirming that the best performance is obtained by using total uncertainty as the rejection criterion.

\begin{figure*}[ht]
  \centering
  \includegraphics[width=.95\linewidth]{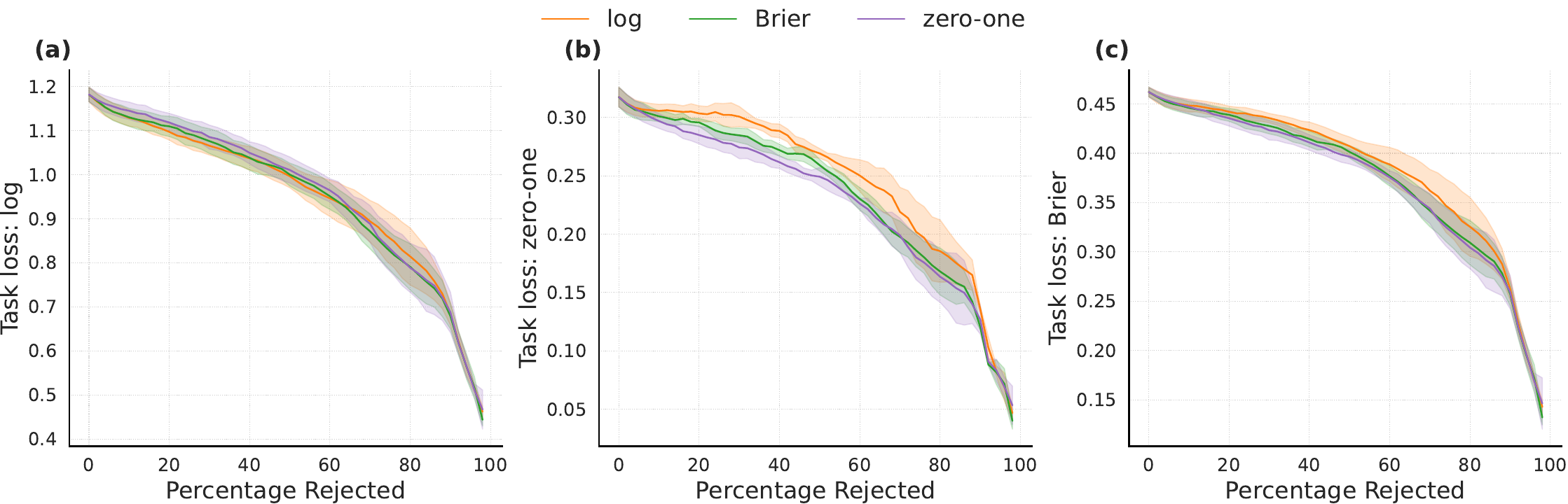}
  \includegraphics[width=.95\linewidth]{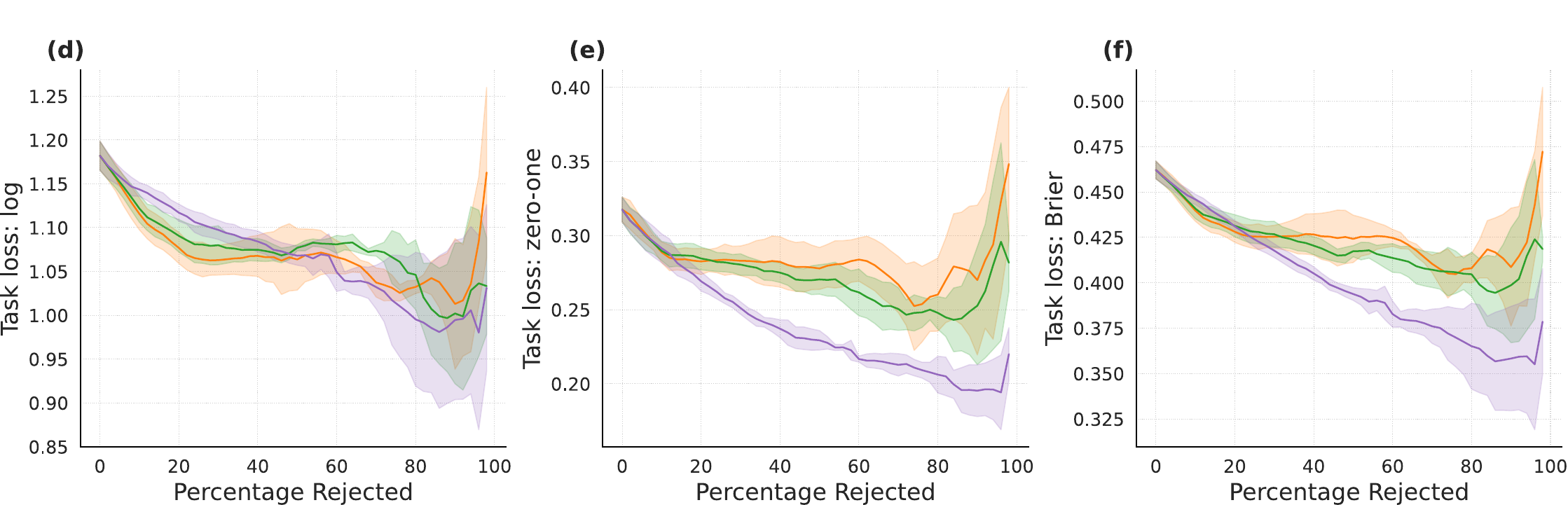}
\caption{Selective prediction with different task losses using the aleatoric (\textbf{top}) and epistemic (\textbf{bottom}) uncertainty component as the rejection criterion where \textbf{(a, d)} uses the \emph{log loss} as task loss, \textbf{(b, e)} \emph{zero-one loss}, and \textbf{(c, g)} the \emph{Brier loss}, respectively. \emph{Results are averaged over three runs.}}
\label{fig:app-sp-au-eu-covtype}
\end{figure*}

We also present selective prediction results for the \dataset{Poker Hand} dataset with total, aleatoric, and epistemic in \cref{fig:app-sp-tu-au-eu-pokerhand}. These results show the same behavior as for the \dataset{CoverType} dataset: The uncertainty loss should be aligned with the task loss and instances should be rejected based on the corresponding total uncertainty measure. The results are summarized in \cref{tab:aulc-cover,tab:aulc-poker}, respectively, which show the area under the loss curve values for all combinations of task losses and uncertainty losses.

\begin{figure*}[h!]
  \centering
  \includegraphics[width=.95\linewidth]{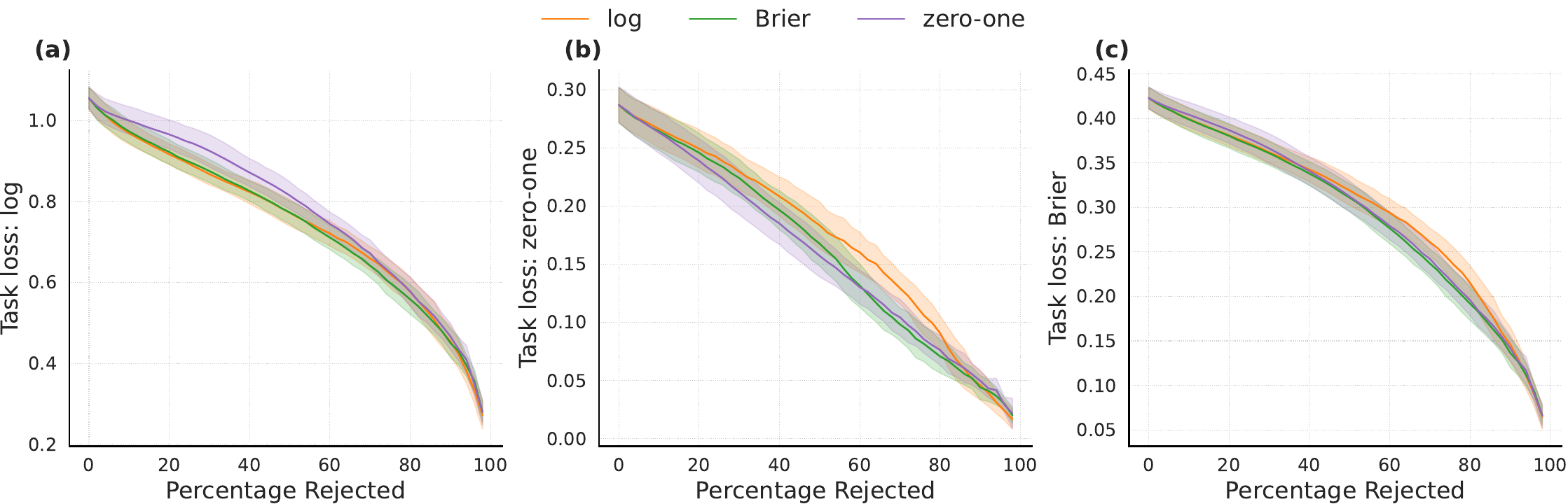}
  \includegraphics[width=.95\linewidth]{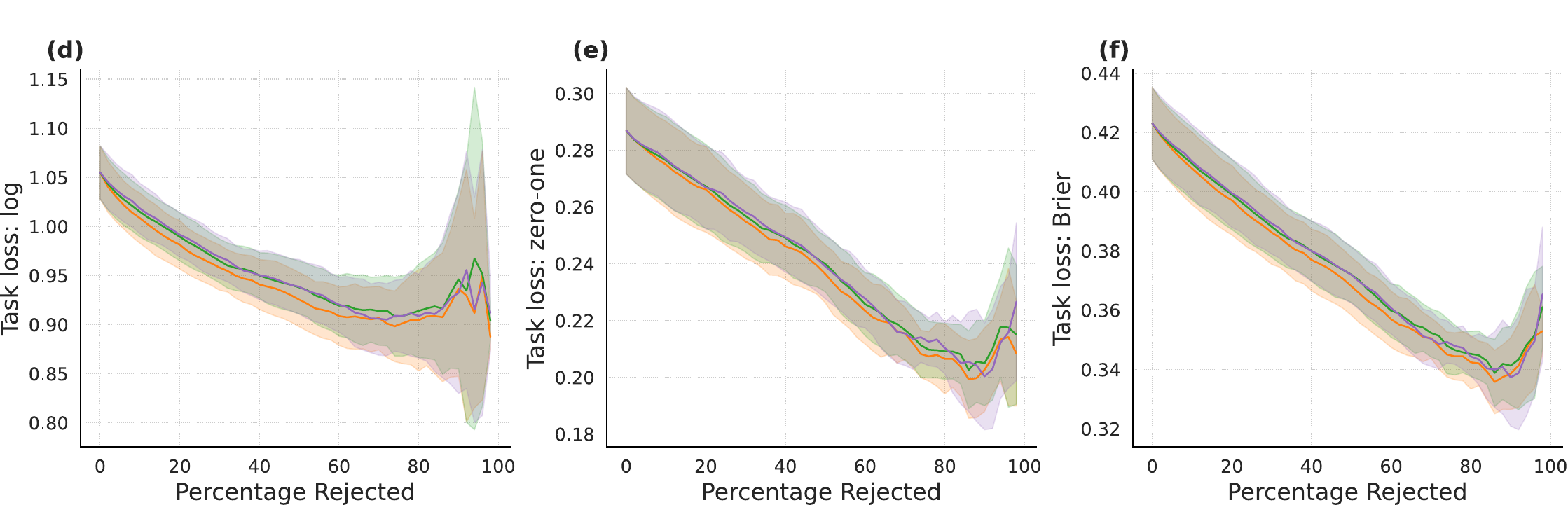}
  \includegraphics[width=.95\linewidth]{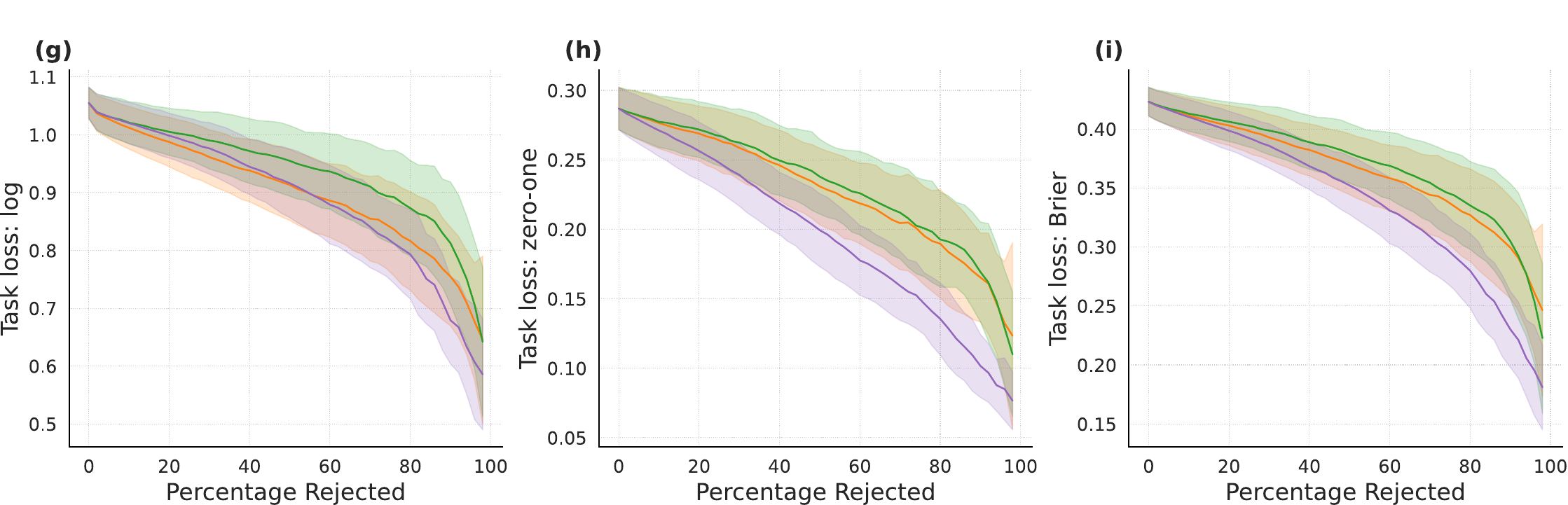}
\caption{Selective prediction with different task losses using the total (\textbf{top}), aleatoric (\textbf{middle}), and epistemic (\textbf{bottom}) uncertainty component as the rejection criterion where \textbf{(a, d, g)} uses the \emph{log loss} as task loss, \textbf{(b, e, h)} \emph{zero-one loss}, and \textbf{(c, g, i)} the \emph{Brier loss}, respectively. \emph{Results are averaged over three runs.}}
\label{fig:app-sp-tu-au-eu-pokerhand}
\end{figure*}

\begin{table}[h]
\begin{tabularx}{\textwidth}{lXXXXXXXXX}
\toprule
         & \multicolumn{3}{c}{total}                                                          & \multicolumn{3}{c}{aleatoric} & \multicolumn{3}{c}{epistemic} \\ \midrule
         & \multicolumn{1}{c}{log} & \multicolumn{1}{c}{Brier} & \multicolumn{1}{c}{zero-one} & log    & Brier   & zero-one   & log    & Brier   & zero-one   \\
log      & $ 0.916 \scriptstyle{ \pm 0.021 }$ & $ 0.915 \scriptstyle{ \pm 0.018 }$ & $ 0.942 \scriptstyle{ \pm 0.017 }$ & $ 0.938 \scriptstyle{ \pm 0.030 }$ & $ 0.936 \scriptstyle{ \pm 0.028 }$ & $ 0.945 \scriptstyle{ \pm 0.030 }$ & $ 1.047 \scriptstyle{ \pm 0.011 }$ & $ 1.053 \scriptstyle{ \pm 0.016 }$ & $ 1.047 \scriptstyle{ \pm 0.025 }$ \\
zero-one & $ 0.234 \scriptstyle{ \pm 0.003 }$ & $ 0.217 \scriptstyle{ \pm 0.003 }$ & $ 0.204 \scriptstyle{ \pm 0.004 }$ & $ 0.243 \scriptstyle{ \pm 0.008 }$ & $ 0.230 \scriptstyle{ \pm 0.006 }$ & $ 0.225 \scriptstyle{ \pm 0.007 }$ & $ 0.276 \scriptstyle{ \pm 0.011 }$ & $ 0.266 \scriptstyle{ \pm 0.005 }$ & $ 0.233 \scriptstyle{ \pm 0.001 }$ \\
Brier    & $ 0.368 \scriptstyle{ \pm 0.008 }$ & $ 0.358 \scriptstyle{ \pm 0.005 }$ & $ 0.356 \scriptstyle{ \pm 0.005 }$ & $ 0.376 \scriptstyle{ \pm 0.012 }$ & $ 0.368 \scriptstyle{ \pm 0.010 }$ & $ 0.366 \scriptstyle{ \pm 0.010 }$ & $ 0.417 \scriptstyle{ \pm 0.002 }$ & $ 0.412 \scriptstyle{ \pm 0.008 }$ & $ 0.391 \scriptstyle{ \pm 0.008 }$ \\ \bottomrule
\end{tabularx}
\caption{Area under the loss curve (AULC) for \dataset{CoverType} with different task losses using the total, aleatoric, and epistemic uncertainty component as the rejection criterion. \textit{The mean and standard deviation of the AULC over three runs are reported.}}
\label{tab:aulc-cover}
\end{table}

\begin{table}[h]
\begin{tabularx}{\textwidth}{lXXXXXXXXX}
\toprule
         & \multicolumn{3}{c}{total}                                                          & \multicolumn{3}{c}{aleatoric} & \multicolumn{3}{c}{epistemic} \\ \midrule
         & \multicolumn{1}{c}{log} & \multicolumn{1}{c}{Brier} & \multicolumn{1}{c}{zero-one} & log    & Brier   & zero-one   & log    & Brier   & zero-one   \\
log      & $ 0.736 \scriptstyle{ \pm 0.030 }$ & $ 0.734 \scriptstyle{ \pm 0.029 }$ & $ 0.764 \scriptstyle{ \pm 0.033 }$ & $ 0.925 \scriptstyle{ \pm 0.031 }$ & $ 0.935 \scriptstyle{ \pm 0.033 }$ & $ 0.934 \scriptstyle{ \pm 0.031 }$ & $ 0.884 \scriptstyle{ \pm 0.060 }$ & $ 0.919 \scriptstyle{ \pm 0.062 }$ & $ 0.875 \scriptstyle{ \pm 0.056 }$ \\
zero-one & $ 0.171 \scriptstyle{ \pm 0.015 }$ & $ 0.159 \scriptstyle{ \pm 0.015 }$ & $ 0.156 \scriptstyle{ \pm 0.015 }$ & $ 0.233 \scriptstyle{ \pm 0.009 }$ & $ 0.235 \scriptstyle{ \pm 0.009 }$ & $ 0.236 \scriptstyle{ \pm 0.009 }$ & $ 0.224 \scriptstyle{ \pm 0.027 }$ & $ 0.228 \scriptstyle{ \pm 0.027 }$ & $ 0.193 \scriptstyle{ \pm 0.022 }$ \\
Brier    & $ 0.295 \scriptstyle{ \pm 0.015 }$ & $ 0.286 \scriptstyle{ \pm 0.015 }$ & $ 0.290 \scriptstyle{ \pm 0.015 }$ & $ 0.364 \scriptstyle{ \pm 0.008 }$ & $ 0.366 \scriptstyle{ \pm 0.008 }$ & $ 0.366 \scriptstyle{ \pm 0.008 }$ & $ 0.357 \scriptstyle{ \pm 0.027 }$ & $ 0.363 \scriptstyle{ \pm 0.027 }$ & $ 0.333 \scriptstyle{ \pm 0.023 }$ \\ \bottomrule
\end{tabularx}
\caption{Area under the loss curve (AULC) for \dataset{Poker Hand} with different task losses using the total, aleatoric, and epistemic uncertainty component as the rejection criterion. \textit{The mean and standard deviation of the AULC over three runs are reported.}}
\label{tab:aulc-poker}
\end{table}

\clearpage 

\newpage 
\subsection{Out-of-Distribution Detection}\label{app:ood}
We present additional out-of-distribution results for the \dataset{ImageNet} and \dataset{Food101} datasets with different uncertainty representations in \cref{tab:app-ood-imagenet} and \cref{tab:app-ood-food101}, respectively. This confirms that mutual information, the log-based epistemic uncertainty measure performs best for the out-of-distribution downstream task. 

\begin{table*}[h!]
\centering
\begin{tabularx}{\textwidth}{XX>{\centering\arraybackslash}X>{\centering\arraybackslash}X>{\centering\arraybackslash}X}
\toprule
Dataset                      & Method & log & Brier & zero-one \\ \midrule
\multirow{2}{*}{\dataset{ImageNet-O}} & Dropout & $ \textbf{0.711} \scriptstyle{\pm 0.009} $ & $ 0.688 \scriptstyle{\pm 0.008} $ & $ 0.550 \scriptstyle{\pm 0.006} $ \\
                            & Laplace & $ \textbf{0.789} \scriptstyle{\pm 0.006} $ & $ 0.713 \scriptstyle{\pm 0.005} $ & $ 0.678 \scriptstyle{\pm 0.008} $ \\\midrule
\multirow{2}{*}{\dataset{CIFAR-100}} & Dropout & $ \textbf{0.876} \scriptstyle{\pm 0.002} $ & $ 0.753 \scriptstyle{\pm 0.002} $ & $ 0.721 \scriptstyle{\pm 0.002} $ \\
                            & Laplace & $ \textbf{0.935} \scriptstyle{\pm 0.001} $ & $ 0.892 \scriptstyle{\pm 0.002} $ & $ 0.894 \scriptstyle{\pm 0.002} $ \\\midrule
\multirow{2}{*}{\dataset{Places365}}  & Dropout     & $ \textbf{0.809} \scriptstyle{\pm 0.001} $ & $ 0.744 \scriptstyle{\pm 0.001} $ & $ 0.671 \scriptstyle{\pm 0.001} $ \\
                            & Laplace & $ \textbf{0.811} \scriptstyle{\pm 0.001} $ & $ 0.732 \scriptstyle{\pm 0.003} $ & $ 0.780 \scriptstyle{\pm 0.002} $ \\\midrule
\multirow{2}{*}{\dataset{SVHN}}       & Dropout     & $ \textbf{0.969} \scriptstyle{\pm 0.001} $ & $ 0.580 \scriptstyle{\pm 0.002} $ & $ 0.857 \scriptstyle{\pm 0.002} $ \\
                            & Laplace & $ \textbf{0.994} \scriptstyle{\pm 0.000} $ & $ 0.956 \scriptstyle{\pm 0.001} $ & $ 0.983 \scriptstyle{\pm 0.001} $ \\ \bottomrule 
\end{tabularx}
\caption{OoD detection with \dataset{ImageNet} as in-Distribution data based on epistemic uncertainty. \emph{The mean and standard deviation of the AUROC over three runs are reported.} Best results are highlighted in \textbf{bold}.}
\label{tab:app-ood-imagenet}
\end{table*}

\begin{table*}[h!]
\centering
\begin{tabularx}{\textwidth}{XX>{\centering\arraybackslash}X>{\centering\arraybackslash}X>{\centering\arraybackslash}X}
\toprule
Dataset                      & Method & log & Brier & zero-one \\ \midrule
\multirow{2}{*}{\dataset{CIFAR-100}}  & Dropout & $ \textbf{0.990} \scriptstyle{\pm 0.000} $ & $ 0.802 \scriptstyle{\pm 0.002} $ & $ 0.921 \scriptstyle{\pm 0.001} $ \\
                            & Laplace & $ \textbf{0.998} \scriptstyle{\pm 0.000} $ & $ 0.996 \scriptstyle{\pm 0.000} $ & $ 0.997 \scriptstyle{\pm 0.000} $ \\\midrule
\multirow{2}{*}{\dataset{Places365}}  & Dropout & $ \textbf{0.987} \scriptstyle{\pm 0.000} $ & $ 0.803 \scriptstyle{\pm 0.002} $ & $ 0.917 \scriptstyle{\pm 0.002} $ \\
                            & Laplace & $ \textbf{0.996} \scriptstyle{\pm 0.000} $ & $ 0.993 \scriptstyle{\pm 0.000} $ & $ 0.994 \scriptstyle{\pm 0.000} $ \\\midrule
\multirow{2}{*}{\dataset{SVHN}}       & Dropout & $ \textbf{0.971} \scriptstyle{\pm 0.000} $ & $ 0.785 \scriptstyle{\pm 0.003} $ & $ 0.934 \scriptstyle{\pm 0.001} $ \\
                            & Laplace & $ \textbf{0.999} \scriptstyle{\pm 0.000} $ & $ 0.995 \scriptstyle{\pm 0.000} $ & $ 0.996 \scriptstyle{\pm 0.000} $ \\ \bottomrule 
\end{tabularx}
\caption{OoD detection with \dataset{Food101} as in-Distribution data based on epistemic uncertainty. \emph{The mean and standard deviation of the AUROC over three runs are reported.} Best results are highlighted in \textbf{bold}.}
\label{tab:app-ood-food101}
\end{table*}

\subsection{Active Learning}\label{app:active-learning}
We present active learning results with additional datasets in \cref{fig:app-al}, confirming the good performance of the zero-one-based epistemic uncertainty measure for this task. 

\begin{figure*}[h!]
  \centering
  \includegraphics[width=.95\linewidth]{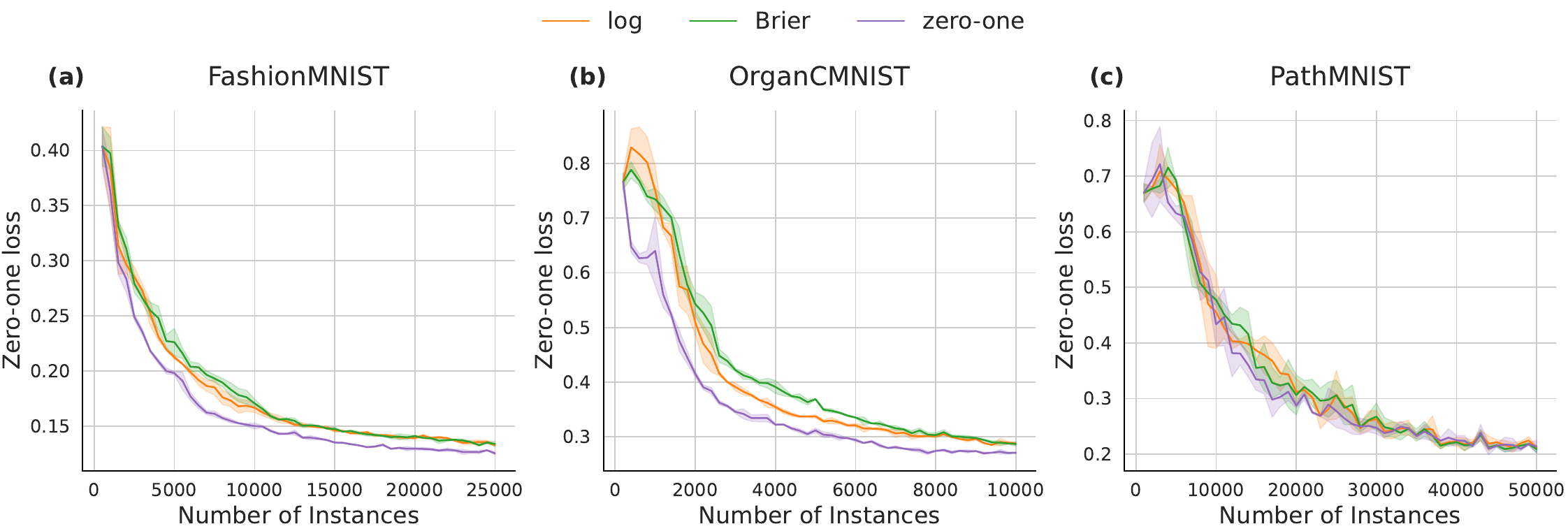}
\caption{Active learning with different datasets using the epistemic uncertainty component to query new instances, where \textbf{(a)} is based on the \dataset{FashionMNIST} dataset, \textbf{(b)} \dataset{OrganCMNIST}, and \textbf{(c)} \dataset{PathMNIST}. \emph{Results are averaged over three runs.}}
\label{fig:app-al}
\end{figure*}

\end{document}